\documentclass{article}

\usepackage{times}
\usepackage{graphicx} 
\usepackage{subfigure} 

\usepackage{natbib}

\usepackage{algorithm}
\usepackage{algorithmic}

\usepackage{amssymb}
\usepackage{latexsym}
\usepackage{amsbsy}
\usepackage{amsmath}

\usepackage{enumerate}

\usepackage{upgreek}
\usepackage{booktabs}
\usepackage{xspace}
\usepackage{mathtools}
\usepackage{gensymb}
\usepackage{amsthm}
\usepackage{commands_dalc}

\usepackage{hyperref}

\renewcommand\rouge[1]{#1}

\usepackage[accepted]{icml2016}

\icmltitlerunning{A New PAC-Bayesian Perspective on Domain Adaptation}

\begin{document} 
	
	\twocolumn[
	\icmltitle{A New PAC-Bayesian Perspective on Domain Adaptation}
	
	\icmlauthor{Pascal Germain}{pascal.germain@inria.fr}
	\icmladdress{INRIA, SIERRA Project-Team, 75589, Paris, France, and D.I.,  Ecole Normale Superieure, 75230 Paris, France}
	\icmlauthor{Amaury Habrard}{amaury.habrard@univ-st-etienne.fr}
	\icmladdress{Univ Lyon, UJM-Saint-Etienne, CNRS, IOGS,
          Laboratoire Hubert Curien UMR 5516, F-42023, Saint-Etienne, France}
	\icmlauthor{Fran\c cois Laviolette}{francois.laviolette@ift.ulaval.ca}
	\icmladdress{D{\'e}partement d'informatique et de g{\'e}nie logiciel, Universit{\'e} Laval, Qu{\'e}bec, Canada}
	\icmlauthor{Emilie Morvant}{emilie.morvant@univ-st-etienne.fr}
	\icmladdress{Univ Lyon, UJM-Saint-Etienne, CNRS, IOGS, 
          Laboratoire Hubert Curien UMR 5516, F-42023, Saint-Etienne, France}
	
	\icmlkeywords{transfer learning, domain adaptation, PAC-Bayesian theory, linear classifiers}
	
	\vskip 0.3in
	]
	
	\begin{abstract}
	  We study the issue of PAC-Bayesian domain adaptation: 
	  We want to learn, from a source domain, a majority vote model dedicated to a target one. 
Our theoretical contribution brings a new perspective 
by deriving an upper-bound on the target risk where the distributions' divergence---expressed as a ratio---controls the trade-off between a source error measure and the target voters' disagreement. 
Our bound suggests that one has to focus on regions where the source data is informative.
From this result, we derive a PAC-Bayesian generalization bound, and specialize it to linear classifiers.
 Then, we infer a learning algorithm
and perform experiments on real data.
	\end{abstract}

	

	\hyphenation{mi-ni-mi-za-tion}
	\hyphenation{mi-ni-mi-zing}
	\hyphenation{me-tho-do-lo-gy}
	\hyphenation{A-da-Boost}
	
	\newtheorem{theorem}{Theorem}
	
	\newtheorem{cor}[theorem]{Corollary}
	
	\newtheorem{definition}[theorem]{Definition}
	\newtheorem{lemma}[theorem]{Lemma}
	\newtheorem{propo}[theorem]{Proposition}
	\newtheorem{remark}[theorem]{Remark}
	
	
	\section{Introduction}
	\label{sec:intro}
	\label{sec:introduction}

	Machine learning practitioners are commonly  
	exposed to the issue of \emph{domain adaptation}\footnote{Domain adaptation is associated with \emph{transfer learning}~\cite{pan2010survey,quionero2009dataset}.}~\cite{jiang2008literature,margolis2011literature}:
	One usually learns a model from a corpus, {\it i.e.}, a fixed yet unknown source distribution, then wants to apply it on a new corpus, {\it i.e.}, a related but slightly different target distribution. 
	Therefore, domain adaptation is  widely studied in a lot of application fields like computer vision~\cite{patel2015visual,ganin-15}, bioinformatics~\cite{liu2008evigan}, natural language processing~\cite{blitzer2007domain,daume07easyadapt}, etc.
	A common
        example is the  spam filtering problem where a model needs to be adapted from one user mailbox to another receiving significantly different emails.
	Many approaches exist 
	to address domain adaptation, 
        often with the same idea: If we 
can
        apply a transformation 
        to ``move closer'' the distributions, then we can learn a model with the available labels.
	This 
        is generally performed 
    by reweighting the importance of labeled data~\cite{HuangSGBS-nips06,Sugiyama-NIPS07,CortesMM10,cortes-15}, 
	and/or by learning a 
        common representation for the source and target distributions \citep{chen-12,ganin-16},
	and/or by minimizing a measure of divergence between the distributions~\cite{morvant12,pbda,CortesM14}. 
	The divergence-based approach has especially been explored to derive generalization bounds for domain adaptation~\citep[\textit{e.g.,}][]{BDnips,BDjournal,Mohri,LiB07,zhang2012generalization}.
	Recently, this issue has been studied through the PAC-Bayesian framework~\cite{pbda},
	which focuses on learning 
	weighted majority votes\footnote{This setting is not too restrictive since many algorithms can be seen as a majority vote learning. {\it E.g.,} ensemble learning and kernel methods output models 
interpretable
          as majority votes.} without 
 target label. Even the latter result opened the door to tackle domain adaptation in a PAC-Bayesian fashion, it shares the same philosophy as the seminal works of \citet{BDnips,BDjournal,Mohri}: The risk of the target model is upper-bounded
jointly by 
the model's risk 
on the source distribution, the divergence between the marginal distributions,
	and a non-estimable term\footnote{More precisely, this term can only be estimated in the presence of labeled data from both the source and the target domains.} related to the ability to adapt in the current space.
	Note that \citet{LiB07} 
        proposed a PAC-Bayesian generalization bound for domain adaptation but they considered target labels. 

In this paper, we derive a novel domain adaptation bound 
          for the weighted majority vote framework. 
Concretely, the risk of the target model is still upper-bounded by three terms, but they differ in the information they capture. The first term is estimable from unlabeled data and relies on a notion of expected voters' disagreement on the target domain.
The second term depends on the expected accuracy of the voters on the source domain. Interestingly, this latter is weighted by a divergence between the source and the target domains that enables controlling the relationship between domains.
The third term estimates the ``volume'' of the target domain living apart from the source one\footnote{Here we do not focus on learning a new representation to help the adaptation: We directly  aim at adapting in the current 
  space.}, which has to be small for ensuring adaptation.
From our bound, we deduce that a good adaptation strategy consists 
in finding a weighted majority vote leading to a suitable trade-off---controlled by the domains' divergence---between the first two terms: 
Minimizing the first one corresponds to look for voters that disagree on the target domain, and  minimizing the second one to seek  accurate voters on the source.
Thereafter, we provide PAC-Bayesian generalization guarantees to justify the empirical minimization of our new domain adaptation bound,
and specialize it to linear classifiers (following a methodology known  to give rise to 
tight bound values).  This allows to design \algo, a learning algorithm  that improves the performances of the  previous PAC-Bayesian domain adaptation algorithm.
	
		The rest of the paper is organized as follows. Section~\ref{sec:notations} presents the PAC-Bayesian domain adaptation setting. Section~\ref{section:previous_bounds} reviews previous theoretical results on domain adaptation.
 Section~\ref{sec:new_bound} states our new analysis of domain adaptation for majority votes, that we relate to other works in Section~\ref{section:comp}. Then,  Section~\ref{section:pacbayes} provides generalization bounds, specialized to linear classifiers in Section~\ref{section:so_specialized} to motivate the \algo learning algorithm, evaluated in Section~\ref{sec:expe}.

	\section{Unsupervised Domain Adaptation Setting}
	\label{sec:notations}
	\label{sec:da}
	
	We tackle \emph{domain adaptation} for \emph{binary classification}, from a \mbox{$d$-dimensional} input space \mbox{$\X\!\subseteq\!\R^d$} to an output space \mbox{$\Y\!=\!\{-1,1\}$}.
        Our goal is to perform 
domain adaptation from a distribution~$\PS$---the \emph{source domain}---to another (related) distribution $\PT$---the \emph{target domain}---on $\XY$; $\DS$ and $\DT$ being the associated marginal distributions on $\X$.
	Given a distribution $\P$, we denote  $(\P)^m$  the distribution of a $m$-sample constituted by $m$ elements drawn \iid\ from  $\P$.
	We consider  the \emph{unsupervised domain adaptation} setting in which the 
	algorithm is provided with a \emph{labeled source $\ms$-sample} \mbox{$S\! =\! \{(\xbf_i,y_i)\}_{i=1}^{\ms}\!\sim\!(\PS)^\ms$},
	and with an \emph{unlabeled target $\mt$-sample} \mbox{$T\! =\! \{\tbfi\}_{i=1}^{\mt}\!\sim\!(\DT)^\mt$}. \\
{\bf PAC-Bayesian domain adaptation.}
Our work 
        is inspired by the PAC-Bayesian theory \citep[first introduced by][]{Mcallester99a}.
	More precisely, we adopt the PAC-Bayesian domain adaptation setting previously studied in \citet{pbda}.
	Given $\Hcal$, a set of voters \mbox{$h : \X \to \Y$}, the 
	\mbox{elements}
        of this 
	approach are a \emph{prior} distribution $\prior$ on $\Hcal$, a pair of source-target learning samples $(S,T)$ and a \emph{posterior} distribution $\posterior$ on $\Hcal$. The prior distribution $\prior$ models an {\it a priori} belief---before observing $(S,T)$---of the voters' accuracy. Then, given the information provided by  $(S,T)$, 
        we aim at learning a posterior distribution $\posterior$ leading to a \emph{\mbox{$\posterior$-weighted} majority vote}  over~$\Hcal$,\\[1mm]
\centerline{$\displaystyle
\BQ(\cdot) \ = \ \sign\left[\esp{h\sim\posterior} h(\cdot)\right],
$}\\[1mm]
	with nice generalization guarantees on the target domain $\PT$.
	In other words, we want to find  the posterior distribution
        $\posterior$ minimizing the true target risk of $\BQ$\,:\\[2mm]
\centerline{$\displaystyle	\RPT(\BQ) \ = \esp{\exbft} \I{\BQ(\tbf)\ne y},$} 
	where $\I{a}=1$ if $a$ is true, and $0$ otherwise.
	However, in most PAC-Bayesian analyses
	one does not directly focus on this majority vote risk,  but studies the expectation of the risks over~$\Hcal$ according to~$\posterior$, designed as the \emph{Gibbs risk}\,:
	\begin{equation} \label{eq:gibbs_risk}
	\Risk_\P(\GQ) \ =\,  \Esp_{\exbf\sim\P}\, \Esp_{h\sim\posterior}\, \I{h(\xbf)\ne y}\,.
	\end{equation}
	It is well-known in the PAC-Bayesian literature that
	$\Risk_\P(\BQ)   \leq  2 \,\Risk_\P(\GQ)$ \citep[\eg,][]{herbrich-00}.
	Unfortunately, this worst case bound often leads to poor generalization guarantees on the majority vote risk. To address this issue, 
	\citet{graal-nips06-mv} \citep[refined in][]{graal-neverending}
	have exhibited that one can obtain a tighter bound on $\Risk_\P(\BQ)$ by studying the \emph{expected disagreement} $\dD(\posterior)$
of pairs of voters, defined as
	\begin{align}
	\dD(\posterior) &=\hspace{-1mm} 
	\Esp_{\xbf\sim \D } 
	\Esp_{h\sim \posterior} \Esp_{h'\sim \posterior}   \I{h(\xbf)\ne h'(\xbf)}, \label{eq:voters_dis}
	\end{align}
	as $\Risk_\P(\BQ)  \leq 1{-}\frac{(1-2\,\Risk_\P(\GQ))^2}{1-2\,\dD(\posterior)}$.
		Note that, although relying on $\dD(\posterior)$, 
		 our present work does not reuse the latter result.\footnote{The quantity $\dD(\posterior)$ is also used in the domain adaptation bound of~\citet{pbda} to measure divergence between distributions. See forthcoming Theorem~\ref{theo:pbda}.}
Instead, we adopt another well-known strategy to obtain tight majority vote bounds, by specializing our PAC-Bayesian bound to linear classifiers. We describe this approach, and refer to related works, in Section~\ref{section:so_specialized}.

\section{Some Previous Domain Adaptation Bounds}
\label{section:previous_bounds}
	
Many approaches tackling domain adaptation 
share the same underlying ``philosophy'', 
pulling its origins
 in the work of \citet{BDnips,BDjournal} which proposed a domain adaptation bound  (Theorem~\ref{theo:BenDavid}, below).
To summarize, the domain adaptation bounds reviewed in this section \citep[see][for other bounds]{zhang2012generalization,CortesMM10,cortes-15} express a similar trade-off between three terms: {\bf (i)} the source risk, {\bf (ii)} the distance between source and target marginal distributions over $\X$, {\bf (iii)} a non-estimable term (without target label) quantifying the difficulty of the 
  task.
\\
 \citet{BDnips} assumed that the domains are related in the sense that there exists a (unknown) model performing well on both domains. 
Formally, their domain adaptation bound depends on 
the error \mbox{$\mu_{h^*}\! = \!\RPS(h^{ *})\!+\!\RPT(h^{ *})$} of the best hypothesis overall
        \mbox{$h^{ *} \! = \!  \argmin_{h\in{\cal H}}  \big( \RPS (h) + \RPT (h) \big)$}. 
In practice, when no target label is available, $\mu_{h^*}$ is non-estimable and is assumed to be low when domain adaptation is achievable (or at least that there exists a representation space in which this assumption can be verified).
In such a scenario, the domain adaptation strategy is then to 
        look for a set $\Hcal$ of possible models that behave ``similarly'' on both the source and target data, and 
        to learn a model in $\Hcal$ with a good accuracy on the source data.
%
	This similarity, called
	the  \mbox{\emph{{\small $\hdh$}-distance}},\\[1mm] 
 $	d_{\hdh}(\DS,\DT) =$\\$\displaystyle 2\!\! \!\sup_{{(h,h')\in\Hcal^2}}  \Big| \Esp_{\xbf\sim\DS}\!\I{h(\xb) \neq h'(\xbf) }\!	-\!\!\!\Esp_{\xbf\sim\DT}\!\I{h(\xb) \neq h'(\xbf) }\Big|,$\\[1mm]
gives rise to the following domain adaptation bound.
	\begin{theorem}[\citealp{BDnips,BDjournal}]
		\label{theo:BenDavid}
		Let ${\cal H}$ be a (symmetric\footnote{In a symmetric 
                  $\Hcal$, for all $h\in\Hcal$, its inverse $-h$ is also in $\Hcal$.}) hypothesis class. We have, 
		\begin{equation}
		\label{eq:da}
		\forall h \!\in\! \Hcal,\ 
		\RPT(h)\, \leq\, \RPS(h)+\tfrac{1}{2}d_{\hdh}(\DS,\DT) + \mu_{h^*}.
\end{equation}
	\end{theorem}
	Pursuing in the same line of research, \citet{Mohri} generalizes the  {\small $\hdh$}-distance to real-valued loss functions \mbox{$\Lcal:[-1,1]^2\to\Rbb^+$}, to express a similar theorem for regression.
	Their \emph{discrepancy} $\disc_\Lcal(\DS, \DT)$
	is defined as
$\displaystyle	\sup_{{(h,h')\in\Hcal^2}}  \left|\Esp_{\xbf\sim\DS}\Lcal\big(h(\xb), h'(\xbf) \big)-\!\!
	\Esp_{\xbf\sim\DT}\Lcal\big(h(\xb), h'(\xbf) \big)
	\right|.$
	The accuracy of the  \citet{Mohri}'s 
        bound 
       also relies on a non-estimable term 
       assumed to be low when adaptation is achievable. Roughly, this term 
depends on
       the risk of the best target hypothesis 
and its \emph{agreement} with the best source hypothesis on the source domain.
%
	
	Building on 
        previous domain adaptation analyses, \citet{pbda} derived a PAC-Bayesian domain adaptation bound.
	This bound is based on a divergence 
	suitable for PAC-Bayes, \ie,
for the risk of a $\posterior$-weighted majority vote of the voters of $\Hcal$ (instead of a single classifier $h\!\in\!\Hcal$).
This \emph{domain disagreement} $\dis(\DS,\DT)$ is defined as 
	\begin{align}
	\label{eq:domain_disagreement}
	\dis(\DS,\DT)
	&=\  \big|\, \dDS(\posterior) - \dDT(\posterior)\, \big| \,.
	\end{align}
Theorem~\ref{theo:pbda} (below) needs the strong assumption that, in favorable adaptation situations, the learned posterior agrees with the best target one  
$\posteriort^{\! *} \!\!   = \!  \argmin_{\posterior}\! \RPT(\GQ)$. 
	Indeed, it relies on the following non-estimable term:  {\small $\lambda(\posterior) \!=\!\RPT(G_{\posteriort^{ *}}\!)\! +\!  \Esp_{h\sim\posterior}\Esp_{ h'\sim\posteriort^{ *}} \Esp_{\xbfs} \I{h(\xbf){\ne} h'(\xbf)}\! +\! \Esp_{h\sim\posterior}\Esp_{ h'\sim\posteriort^{ *}} \Esp_{\xbft} \I{h(\tbf){\neq} h'(\tbf)}$. }
	
	\begin{theorem}[\citealp{pbda}]
		\label{theo:pbda}
		Let $\Hcal$ be a set of voters. For any domains $\PS$ and $\PT$ over $\XY$, 
		we have, 
\centerline{$\displaystyle		\forall \posterior\mbox{ on }\Hcal,\ 
		\RPT(\GQ)\ \leq\ \RPS(\GQ) + \dis(\DS,\DT) + \lambda(\posterior).$}
	\end{theorem}
	A compelling aspect of this PAC-Bayesian analysis is the suggested trade-off, which is function of~$\posterior$. 
Indeed, 
given a fixed instance space $\X$ and a fixed class $\Hcal$, apart from using importance weighting methods, the only way to minimize the bound of Theorem~\ref{theo:BenDavid} is to find $h\!\in\!\Hcal$ that minimizes $\RPS(h)$.
In~\citet{pbda}, the bound of Theorem~\ref{theo:pbda} inspired an algorithm---named {\sc pbda}---selecting $\posterior$ over $\Hcal$ that achieves a trade-off between $\RPS(\GQ)$ and $\dis(\DS,\DT)$.
		However, the term $\lambda(\posterior)$ 
		does not appear in the optimization process of {\sc pbda}, even if it relies on the learned weight distribution $\posterior$.
		It is assumed that the value of $\lambda(\posterior)$ 
		should be negligible (uniformly for all $\posterior$) when adaptation 
is achievable. Nevertheless, this strong assumption cannot be verified because the best target posterior distribution $\posteriort^{ *}$ is unknown.
		This is a major weakness of the previous PAC-Bayesian work that our new approach overcomes.

	\section{A New Domain Adaptation Perspective}
\label{sec:new_bound}
	In this section, we introduce an original approach to upper-bound the non-estimable risk of a $\posterior$-weighted majority vote  on a target distribution~$\PT$ 
	thanks to a term depending on its marginal distribution~$\DT$, another one on a related source domain $\PS$, and a term capturing  the ``volume'' of the source distribution uninformative for the target task.
We base our bound on	
	the expected disagreement $\dD(\Q)$ of Equation~\eqref{eq:voters_dis} and the expected joint error $\eP(\Q)$, defined as
		\begin{align}
			\eP(\posterior) &=\hspace{-3mm}
			\Esp_{\exbf\sim\P} 
			\Esp_{h\sim \posterior} \Esp_{h'\sim \posterior}   \I{h(\xbf)\ne y}\, \I{h'(\xbf)\ne y}\,. \label{eq:erreur_jointe}
		\end{align} 
	Indeed,  \citet{graal-nips06-mv,graal-neverending} 
	observed that, given a domain~$\P$ on $\XY$ and a distribution $\posterior$ on $\Hcal$, we can decompose the Gibbs risk   as
		\begin{align} \label{eq:GibbsDE}
		\nonumber
		&\Risk_\P(\GQ) \!= \tfrac12 \Esp_{\exbf\sim\P} \Esp_{h\sim\posterior} \Esp_{h'\sim\posterior} \I{h(\xbf){\ne} y}\!{+}\I{h'(\xbf){\ne} y}\\
		\nonumber&=\!\!\!\!\! \Esp_{\exbf\sim\P}\Esp_{h\sim\posterior} \Esp_{h'\sim\posterior}   
		\!\!\frac{\I{h(\xbf){\ne} h'(\xbf)}\!{+}2\,\I{h(\xbf){\ne} y\!\wedge\! h'(\xbf){\ne} y}}{2}\\
		&=\ \tfrac12\,\dD(\Q) + \eP(\posterior) \,.
		\end{align}
%
	A key observation is that the \emph{voters' disagreement does not rely on labels}; we can compute $\dD(\Q)$ using the marginal distribution $\D$. Thus, in the present domain adaptation context, we have access to $\dDT(\Q)$ even if the target labels  are unknown.
	However, the expected joint error can only be computed on the labeled source domain. 

\textbf{Domains' divergence.}
In order to link the target joint error $\ePT(\Q)$ with the source one $\ePS(\Q)$, we weight the latter thanks to a divergence measure between the domains
 $\bq$ parametrized by a real value $q>0$\,: 
	\begin{equation} \label{eq:bq}
	\bq \ = \ \left[\,\esp{\exbfs}
	\left(  \frac{\PT(\xbf,y)}{\PS(\xbf,y)} \right)^q\, \right]^{\frac1q}.
	\end{equation}
	It is worth noting that  considering some $q$ values allow us to recover well-known divergences. For instance,  choosing
        \mbox{$q\! =\! 2$} relates
	our result  to the \mbox{$\chi^2$-distance}, 
         as
	\mbox{$\bq[2] \! = \! \sqrt{ \chi^2(\PT\|\PS)+1}\,.$ }
	Moreover, we can  link $\bq$ to the R{\'e}nyi divergence\footnote{For  $q\geq 0$, we can 
          show 
	\mbox{$\bq\! =\! 2^{\frac{q-1}{q} D_q(\PT\|\PS)}$},
		where $D_q(\PT\|\PS)$ is the R{\'e}nyi divergence between $\PT$ and $\PS$.},  which has 
        led to generalization bounds in the 
        context of importance weighting \citep{CortesMM10}. 
	We denote the limit case $q\!\to\!\infty$ by\\[1mm]
\centerline{$\displaystyle	\binf   \ =  
	\sup_{(\xbf,y)\in\scriptsupport(\PS)}
	\left(  \frac{\PT(\xbf,y)}{\PS(\xbf,y)} \right),$}\\[1mm]
	with $\support(\PS)$ the support of 
        $\PS$.
The divergence $\bq$ handles the input space  areas where the source domain support $\support(\PS)$ is included in the target one $\support(\PT)$. It seems reasonable to assume that, when adaptation is achievable, such areas are fairly large. However, it is 
likely that $\support(\PT)$ is \emph{not entirely} included in $\support(\PS)$.
We denote  $\TminusS$ the distribution of $(\xbf,y){\sim}\PT$ conditional to $(\xbf,y){\in}\support(\PT){\setminus}\support(\PS)$.
Since it is hardly conceivable to estimate the joint error $\e_\TminusS(\posterior)$ without making extra assumptions, we 
define the worst 
risk for this \emph{unknown} area 
	\begin{equation}
	\label{eq:etaTS}
	\ets = \!\!\!
\Pr_{(\xb,y)\sim \PT} \!\Big((\xb,y)\notin \support(\PS)\Big) \ 
	\sup_{h\in\Hcal} R_\TminusS(h)\,.
	\end{equation}
	Even if we cannot evaluate  $\sup_{\Hcal} R_\TminusS(h)$, the value of $\ets$ is necessarily lower than $\Pr_\PT((\xb,y){\notin} \support(\PS))$. 

\textbf{The domain adaptation bound.}
Let us state the result underlying the domain adaptation perspective of this paper.
	\begin{theorem}\label{theo:new_bound_general}
		Let $\Hcal$ be a hypothesis space, let $\PS$ and $\PT$ respectively be the source and the target domains on $\XY$. 
		Let $q>0$ be a constant. 
		We have, for all $\posterior$ on $\Hcal$,
		\begin{align*} 
		\RPT(\GQ) \, \leq \, \frac12 \,\dDT(\Q) +
		\bq{\times }
		\Big[ \ePS(\Q) \Big]^{1-\frac1q}
		\!+ \eta_{\PT\setminus\PS}\,,
		\end{align*}
		where
		$\dDT(\Q)$, $\ePS(\Q)$, $\bq$ and $\ets$ are respectively defined by Equations~\eqref{eq:voters_dis}, \eqref{eq:erreur_jointe}, \eqref{eq:bq} and~\eqref{eq:etaTS}.
	\end{theorem}
	\begin{proof}
Let us define $t\!=\!\Esp_{\exbft}  \I{(\xb,y) {\notin} \support(\PS)} $, then
\begin{align*}
\eta_\posterior&=  \!\!\!\!\!\!\!
\Esp_{\exbft} \!\!\! \I{(\xb,y) {\notin} \support(\PS)}\!  \Esp_{h\sim\posterior} \Esp_{h'\!\sim\posterior}   \!\!
\I{h(\tbf){\ne} y}\I{h'(\tbf){\ne} y} \\
 =& \ t \!\!\!\!\!\!
 \Esp_{{(\xbf,y)\sim\TminusS}}  \Esp_{h\sim\posterior} \Esp_{h'\sim\posterior}   
 \I{h(\tbf){\ne} y}\,\I{h'(\tbf){\ne} y}
 =  t\, \e_\TminusS(\posterior) \\
=& \ t \Big(R_\TminusS(G_\posterior) {-} \tfrac12{\rm d}_{\TminusS}(\posterior) \Big)\leq t \sup_{h\in\Hcal} R_\TminusS(h) =\eta_\TminusS\,.
\end{align*}
		Then, with $\beta_q\!=\!\bq$ and $p$ such that $\,\tfrac1p{=}1{-}\tfrac1q$, 
\allowdisplaybreaks[4]
		\begin{align} 
		&\nonumber\ePT(\posterior) \, =\!\!\!\Esp_{\exbft}  
		\Esp_{h\sim\posterior} \Esp_{h'\sim\posterior}  
		\I{h(\tbf){\ne }y}\,\I{h'(\tbf){\ne} y}
		\\
		\label{eq:encore_egal}
		&= \!\!\!\Esp_{\exbfs}  \! \tfrac{\PT(\xbf,y)}{\PS(\xbf,y)} \!\! \Esp_{h\sim\posterior} \Esp_{h'\sim\posterior}  \!\I{h(\xbf){\ne} y}\I{h'(\xbf){\ne} y}\!+\!\eta_\posterior\\
		\nonumber &\leq   
		\beta_q \!\left[ 
		\Esp_{h\sim\posterior} \Esp_{h'\sim\posterior}  
		 \Esp_{\exbfs} \!\left( \I{h(\xbf){\ne} y}\I{h'(\xbf){\ne} y} \right)^p
		\right]^{\!\frac1p}\!\!\!+\!\eta_\posterior\,.
		\end{align}
		Last line is due to H{\"o}lder inequality. Finally, we remove the exponent from expression $(\I{h(\xbf)\ne y}\I{h'(\xbf)\ne y})^p$ without affecting its value, which is either $1$ or~$0$, and the final result follows from Equation~\eqref{eq:GibbsDE}.
	\end{proof}
		Note that the bound of Theorem~\ref{theo:new_bound_general} is reached whenever the domains are equal ($\PS\! =\! \PT$).
		Thus, when adaptation is not necessary, our analysis is still sound and  non-degenerated: 
		\begin{align*}
		\RPS(\GQ) \ =\ \RPT(\GQ) \ &\leq\  \tfrac12\,
		\dDT(\Q) +  1\times\left[\ePS(\Q)\right]^1+0\\
		&=\  \tfrac12\, \dDS(\Q) +  \ePS(\Q) 
		=\ \RPS(\GQ) \,.
		\end{align*}	
\textbf{Meaningful quantities.}
	Similarly to the previous results recalled in Section~\ref{section:previous_bounds}, our domain adaptation theorem bounds the target risk by a sum of three terms. However, our approach breaks the problem into \emph{atypical} quantities: 
{\bf (i)} The expected disagreement $\dDT(\posterior)$ captures \emph{second degree} information about the target domain.
{\bf (ii)} The domains' divergence $\bq$
weights the influence of the expected joint error $\ePS(\posterior)$ of the source domain; the parameter~$q$ allows us to consider different relationships between $\bq$ and $\ePS(\Q)$. 
{\bf (iii)} The term $\ets$ quantifies the worst feasible target error on the regions 
where the source domain is uninformative for the target one.
In the current work, we assume that this area is small.

\section{Comparison With Related Works}
	\label{section:comp}
In this section, we discuss how our domain adaptation bound can be related to some previous works.
\subsection{On the previous PAC-Bayesian bound}
	It is instructive to compare the new bound of Theorem~\ref{theo:new_bound_general}
        with the previous PAC-Bayesian domain adaptation bound of Theorem~\ref{theo:pbda}. 
	In Theorem~\ref{theo:new_bound_general}, the non-estimable terms are the domain divergence $\bq$ and the  term $\ets$. Contrary to the non-controllable term
	$\lambda(\posterior)$ 
	of Theorem~\ref{theo:pbda}, 
	these terms do not depend on the \emph{learned} posterior distribution~$\posterior$: 
	 For every $\posterior$ on $\Hcal$, 
	$\bq$ and $\ets$
	are constant values measuring the relation between the domains. 
	Moreover, the fact that the domain divergence $\bq$ is not an additive term but a
	multiplicative one (as opposed to {\small$\dis(\DS,\DT)\! +\! \lambda(\posterior)$}
	in Theorem~\ref{theo:pbda}) is a contribution of our new analysis.
	 Consequently, $\bq$ 
	can be viewed as a hyperparameter allowing us to tune the
	trade-off between the target voters'
	disagreement and the source joint
	error. Experiments of Section~\ref{sec:expe} confirm that this hyperparameter can be successfully selected. 
	
\subsection{On some domain adaptation assumptions}
\label{section:assumptions}
In order to characterize which domain adaptation task may be \emph{learnable}, \citet{bendavid-12} presented three \emph{assumptions that can help domain adaptation}. 
Our Theorem~\ref{theo:new_bound_general}
does not rely on these assumptions, 
but they can be interpreted in our framework as discussed below. 
\\[0.5mm]
\textbf{On the covariate shift.} 
A domain adaptation task fulfills the \emph{covariate shift} assumption \cite{covariateshift} if the source and target domains only  differ in their marginals according to the input space, {\it i.e.,}
$\PT_{\Y|\xb}(y) = \PS_{\Y|\xb}(y)$. In this scenario,  
one may estimate 
$\bqx$, and even $\ets$, by using unsupervised density estimation methods.
	Interestingly, by also assuming 
        that the domains share the same support, we have \mbox{$\ets\!=\!0$}. Then from Line~\eqref{eq:encore_egal}  we obtain 

\mbox{\small$\displaystyle \RPT(\GQ) \!= \! \tfrac12 \dDT(\Q) \!   +  \! \! \!
	\Esp_{\xbfs}\!\tfrac{\DT(\xbf)}{\DS(\xbf)}  \!  \Esp_{h\sim\posterior}   \Esp_{h'\!\sim\posterior}\!\!
	\I{h(\xbf){\ne}y}\I{h'(\xbf){\ne}y}\!,$}\\[1mm]
	which suggests a way to correct the \emph{shift} between the domains
	by reweighting the labeled source distribution, while considering the information from the target disagreement.
\\[0.5mm]
\textbf{On the weight ratio.} 
 The \emph{weight ratio} \citep{bendavid-12} of source and target domains, with respect to a collection of input space subsets $\Bcal\subseteq 2^\X$, is given by\\[1mm]
\centerline{$\displaystyle C_\Bcal(\PS, \PT) \ = \ 
\inf_{\substack{b\in\Bcal,\, \DT(b)\neq 0 }} \,
\frac{\DS(b)}{\DT(b)}\,.$}\\[1mm]
When $C_\Bcal(\PS, \PT)$ is bounded away from $0$, adaptation should be achievable under covariate shift.
In this context, and when 
	\mbox{$\support(\PS)\!=\!\support(\PT)$}, 
	the limit case of  $\binf$ is equal to the inverse of the \emph{point-wise weight ratio} obtained by letting $\Bcal\!=\!\{\{\xbf\}:\xbf\in\X\}$ in $C_\Bcal(\PS, \PT)$.
Indeed, both $\beta_q$ and $C_\Bcal$ compare the densities of source and target domains, but provide distinct strategies to relax the point-wise weight ratio; the former by lowering the value of~$q$ and the latter by considering larger subspaces~$\Bcal$.
\\[0.5mm]
\textbf{On the cluster assumption.} A target domain fulfills the \emph{cluster assumption} when examples of the same label belong to a common ``area'' of the input space, 
and the differently labeled ``areas'' are well separated by \emph{low-density regions} \citep[formalized by the \emph{probabilistic Lipschitzness} of][]{urner-11}. 
Once specialized to linear classifiers, $\dDT(\posterior)$ behaves nicely in this context (see Section~\ref{section:so_specialized}).

\subsection{On representation learning}
\label{section:representation_learning}

The main assumption underlying our domain adaptation algorithm exhibited in Section~\ref{section:so_specialized} is that the support of the target domain is mostly included in the support of the source domain, \ie, the value of the term $\ets$ is small.
When $\TminusS$ is sufficiently large to prevent proper adaptation, one could try to reduce its volume  while taking care to preserve a good compromise between $\dDT(\Q)$ and $\ePS(\Q)$, using a \emph{representation learning} approach, 
\ie, by projecting source and target examples into a new common 
space, as done for example by \citet{chen-12,ganin-16}.

	\section{PAC-Bayesian Generalization Guarantees}
	\label{section:pacbayes}
	To compute our domain adaptation bound, one needs to know the distributions $\PS$ and $\DT$, which is never the case in real life tasks. The PAC-Bayesian theory provides tools to convert the bound of Theorem~\ref{theo:new_bound_general} into a generalization bound on the target risk computable from a pair of source-target samples $(S,T){\sim}(\PS)^\ms\!{\times} (\DT)^\mt$. 
	To achieve this, we first provide generalization guarantees for $\dDT(\Q)$ and $\ePS(\Q)$.
	These results are presented as corollaries of Theorem~\ref{theo:catoni_genral} below, that generalizes a PAC-Bayesian theorem of \citet{catoni2007pac} 
to  arbitrary loss functions.\footnote{To do so, we exploit a result of  \citet{maurer-04} that allows to generalize PAC-Bayes theorems to arbitrary bounded loss function (see the proof of Theorem~\ref{theo:catoni_genral} in supplemental).}
	Indeed, Theorem~\ref{theo:catoni_genral}, with $\ell(h,\xb,y)\! =\!
	\I{h(\xb){\neq} y}$ and Equation~\eqref{eq:gibbs_risk}, gives the usual bound
	on the Gibbs risk. 
	\begin{theorem}
		\label{theo:catoni_genral}
		For any domain $\P$ over $\X{\times}\Y$, any set of voters $\Hcal$, any prior  $\prior$ over~$\Hcal$, any loss $\ell:\Hcal{\times}\X{\times}\Y{\to}[0,1]$, any real number $c{>}0$, with a probability at least $1{-}\delta$ over the choice of $\{(\xb_i,y_i)\}_{i=1}^m {\sim} (\P)^m$, we have 
		for all $\posterior$ on $\Hcal$:
		\begin{align*}
		&\Esp_{\exbf\sim\P}  \Esp_{h\sim\posterior} \ell(h, \xb,y)
		\\[-2mm]
		 & \ \leq\
		\frac{c}{1{-}e^{-c}} \left[\frac{1}{m}\sum_{i=1}^m \Esp_{h\sim\posterior} \!\!\ell(h,\xb_i,y_i) + \frac{\KL(\posterior\|\prior) + \ln \frac{1}{\delta}}{ m\times c}\right].
		\end{align*}
	\end{theorem}
Note that, similarly to \citet{mcallester-keshet-11}, we could choose to restrict $c\in(0,2)$ to obtain a slightly looser but simpler bound. Using $e^{-c}\leq1-c-\frac12c^2$, an upper bound on the 
right hand side
          of above equation is given by\\
	$\tfrac{1}{1{-}\frac12 c} \left[\tfrac{1}{m}\sum_{i=1}^m \Esp_{h\sim\posterior} \ell(h,\xb_i,y_i) + \frac{\KL(\posterior\|\prior) + \ln \frac{1}{\delta}}{ m\times c}\right].$


	We now exploit Theorem~\ref{theo:catoni_genral}  to obtain generalization guarantees on the expected disagreement and the expected joint error.
	PAC-Bayesian bounds on these quantities appeared in \citet{
          graal-neverending},
	 but under different forms. 
	In Corollary~\ref{theo:catoni_new} below, 
	we are especially interested in the possibility of controlling the trade-off---between the empirical estimate computed on the samples and the complexity term 
	 $\KL(\posterior\|\prior)$---with the help of parameters $b$ and $c$.
	
	\begin{cor}
		\label{theo:catoni_new}
		For any domains $\PS$ and $\PT$ over $\X{\times}\Y$, any set of voters $\Hcal$, any prior $\prior$ over $\Hcal$, any $\delta{\in}(0,1]$, any real numbers $b > 0$ and $c>0$, we have:\\
		--- with a probability at least $1{-}\delta$ over $T\sim(\DT)^{\mt}$, 
		\vspace{-1.5mm}
		\begin{align*}
		\forall \posterior \mbox{ on }\Hcal,\,
		\dDT(\Q)\, \leq\, \frac{c}{1{-}e^{-c}} \!\left[
	\dT(\Q) \!+\!   \frac{2\KL(\posterior\|\prior) \!+\! \ln \frac{1}{\delta}}{\mt\times c}\right]\!,
		\end{align*}
		--- with a probability at least $1{-}\delta$ over $S\sim(\PS)^{\ms}$,
				\vspace{-1.5mm}
		\begin{align*}
		\forall \posterior \mbox{ on }\Hcal, \,
		\ePS(\Q)\, \leq\, \frac{b}{1{-}e^{-b}} \!\left[
		\eS(\Q) \!+\!   \frac{2\KL(\posterior\|\prior) \!+\! \ln \frac{1}{\delta}}{\ms\times b}\right]\!,
		\end{align*}
where $\dT(\Q)$ and $\eS(\posterior) $ are the empirical estimations of the target voters' disagreement and the source joint error.
	\end{cor}
	\begin{proof}
		Given $\prior$ and $\posterior$ over $\Hcal$, we consider a new prior $\prior^2$ and a new posterior $\posterior^2$, both over $\Hcal^2$, such that: \mbox{$\forall\, h_{ij}  =  (h_i,h_j) \in  \Hcal^2,\ \prior^2(h_{ij})  =  \prior(h_i)\prior(h_j)$,} 
		and \mbox{$\posterior^2(h_{ij})  =  \posterior(h_i)\posterior(h_j)$}.
		Thus, $\KL(\posterior^2\|\prior^2)  =  2 \KL(\posterior^2\|\prior^2)$ \citep[see][]{graal-neverending}.  
		Let us define two new loss functions for a ``paired voter'' $h_{ij}  \in  \Hcal^2$:
		\begin{align*}
		\ell_d(h_{ij}, \xb, y)\ &= \ \I{h_i(\xbf)\ne h_j(\xbf)}\,,\\
	\mbox{and}\ \	\ell_e(h_{ij}, \xb, y)\ &= \ \I{h_i(\xbf)\ne y}{\times} \I{h_j(\xbf)\ne y}\,.
		\end{align*}
		Then, the bound on $\dDT(\Q)$ is obtained from Theorem~\ref{theo:catoni_genral} with $\ell\eqdots \ell_d$, and Equation~\eqref{eq:voters_dis}. The bound on $\ePS(\Q)$ 
		is similarly obtained with $\ell\eqdots \ell_e$ and using Equation~\eqref{eq:erreur_jointe}.
	\end{proof}
	
For algorithmic simplicity, we deal with
	Theorem~\ref{theo:new_bound_general} when
	$q{\to}\infty$. 
	Thanks
	to Corollary~\ref{theo:catoni_new}, we obtain the following generalization bound defined 
	with respect to
	the empirical estimates of the target disagreement and the
	source joint error.
	\begin{theorem}
		\label{theo:catoni_new_general}
		For any domains $\PS$ and $\PT$ over $\X{\times}\Y$, any set of voters $\Hcal$, any prior $\prior$ over $\Hcal$, any $\delta{\in}(0,1]$, any
$b {>} 0$ and $c{>}0$, with a probability at least $1{-}\delta$ over the choices of $S{\sim}(\PS)^\ms$ and $ T {\sim} (\DT)^\mt$, we have
		\begin{align*}
				\forall \posterior \mbox{ on }\Hcal, \,
		\RPT(\GQ) & \leq\ c'\,\tfrac12\,\dT(\Q)  +b'\,\eS(\posterior) + \ets \\
		& {}+\!\left(\tfrac{c'}{\mt\times c} {+} \tfrac{b'}{\ms\times b} \right)  \!\! \Big(2\,\KL(\posterior\|\prior) + \ln \tfrac{2}{\delta}\Big) ,
		\end{align*}
where $\dT(\Q)$ and $\eS(\posterior) $ are the empirical estimations of the target voters' disagreement and the source joint error, and 
		$b'=\frac{b}{1-e^{-b}}\binf$, and 
		$c'=\frac{c}{1-e^{-c}}$.
	\end{theorem}
	\begin{proof}
		We bound  separately $\dDT(\Q)$ and $\ePS(\Q)$ using Corollary~\ref{theo:catoni_new} (with probability $1{-}\frac{\delta}{2}$ each), and then combine the two upper bounds according to Theorem~\ref{theo:new_bound_general}.
	\end{proof}
		From an optimization perspective, the problem suggested by the bound of Theorem~\ref{theo:catoni_new_general} is much more convenient to minimize than the PAC-Bayesian bound derived from  Theorem~\ref{theo:pbda} in \citet{pbda}.
		The former is \emph{smoother} than the latter: the absolute value related to the domain disagreement $\dis(\DS,\DT)$ of Equation~\eqref{eq:domain_disagreement} disappears in benefit of the domain divergence $\binf$, which is constant and can be considered as an hyperparameter of the algorithm. Additionally, Theorem~\ref{theo:pbda} requires equal source and target sample sizes while Theorem~\ref{theo:catoni_new_general} allows $\ms{\neq}\mt$.
		Moreover, recall that in \citet{pbda} the \mbox{$\posterior$-dependent} non-constant term $\lambda(\posterior)$ is ignored. In our new analysis, such compromise is not mandatory in order to apply the theoretical result to real  problems, since the non-estimable term $\ets$ is constant and does not depend on the learned $\posterior$. Hence, we can neglect $\ets$  without any impact on the optimization problem described in the next section. Beside, it is  realistic to consider $\ets$ as a small quantity in situations where the source and target supports are similar.

	\begin{figure*}[t]
		\begin{minipage}[c]{.3\textwidth}\centering
			\includegraphics[width=1\textwidth]{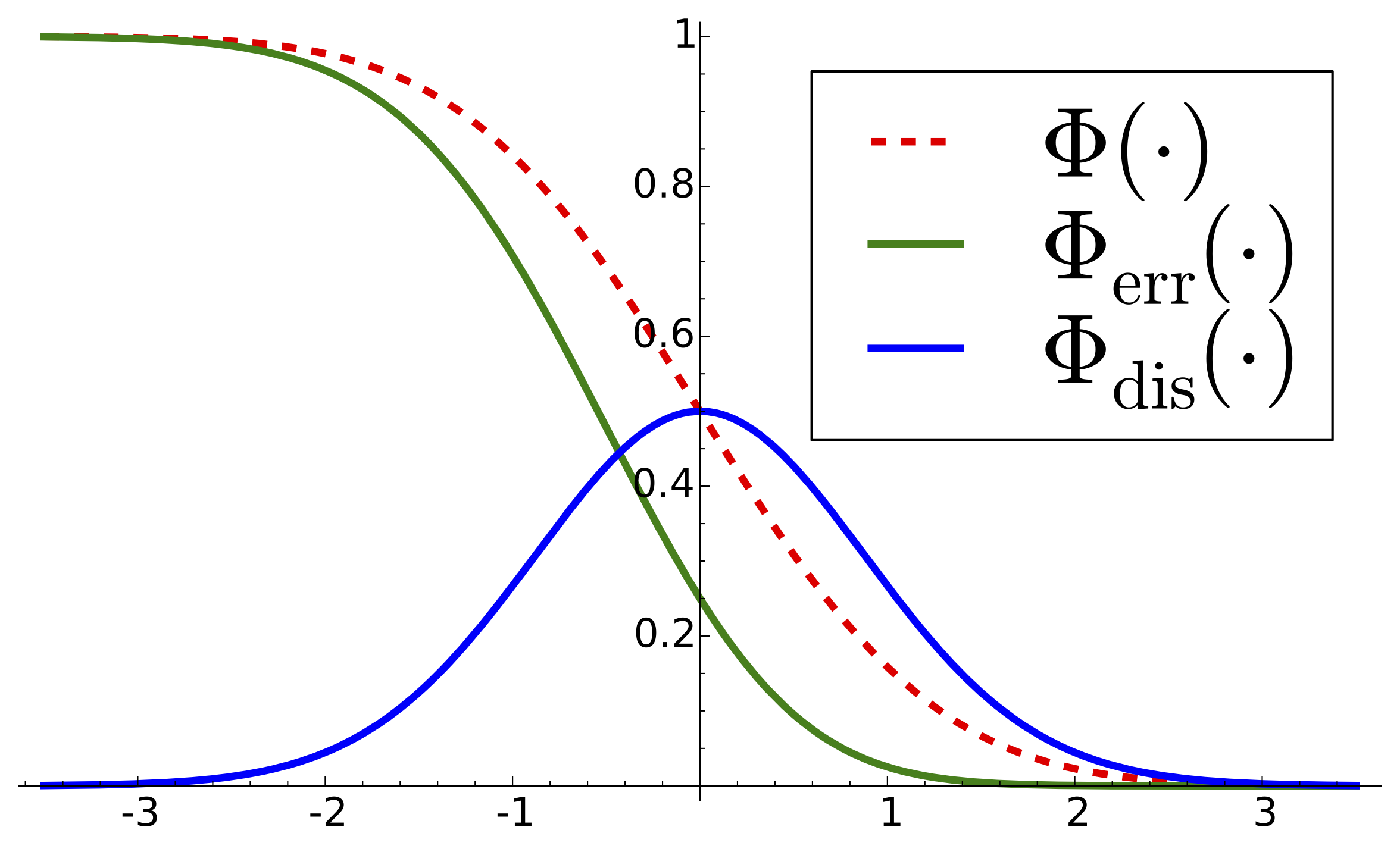}
			\caption{Graphical representation of the loss functions given by the specialization to linear classifiers.
				\label{fig:losses}}
		\end{minipage}\hfill
		\begin{minipage}[c]{.65\textwidth} 
			\centering
			\includegraphics[width=0.32\textwidth,trim=20mm 10mm 13mm 14mm]{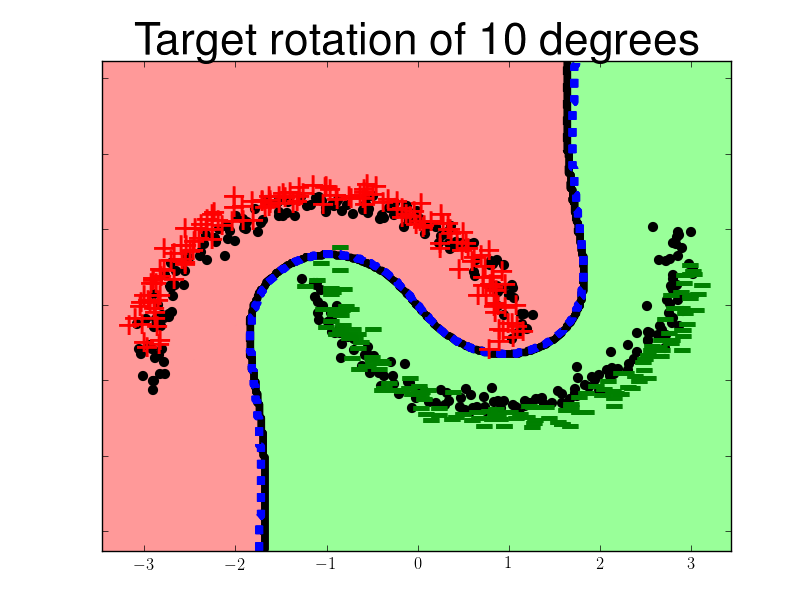}\hfill
			\includegraphics[width=0.32\textwidth,trim=17mm 10mm 17mm 14mm]{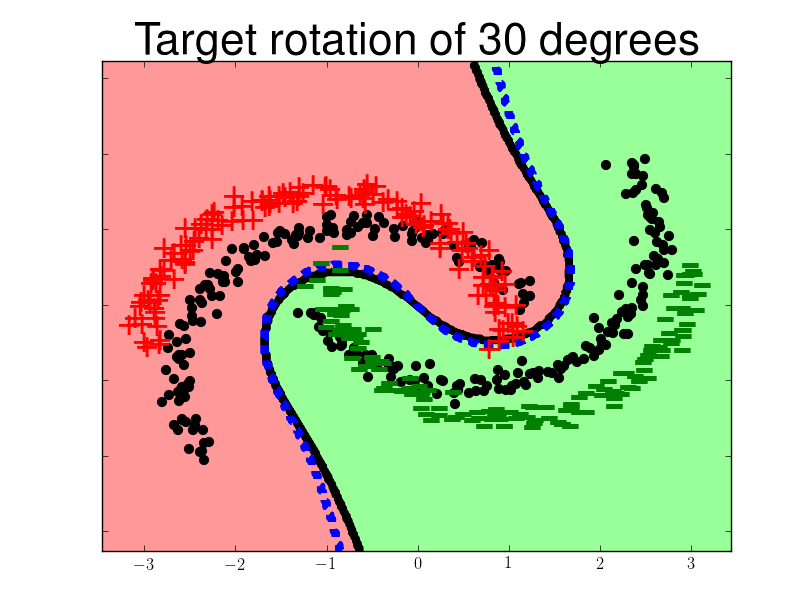}\hfill
			\includegraphics[width=0.32\textwidth,trim=17mm 10mm 17mm 14mm]{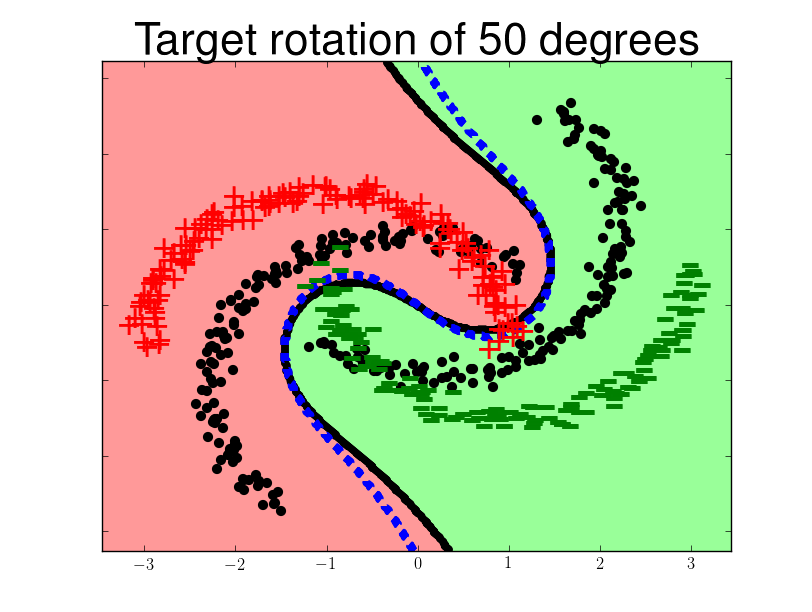}
			\caption{Decision boundaries of \algo~on the \emph{intertwining moons} toy problem, for fixed parameters $B{=}C{=}1$, and a RBF kernel $k(\xb,\xb')=\exp({- \|\xb-\xb'\|^2})$. The target points are black. The positive, {\it resp.} negative, source points are red, {\it resp.} green. The blue dashed line shows the decision boundaries of
				algorithm {\sc pbda}~\cite{pbda}.
				\label{fig:lunes}}
		\end{minipage}
	\end{figure*}
	
\section{Specialization to Linear Classifiers}
\label{section:so_specialized}

In order to derive an algorithm, we now specialize the bounds of Theorems~\ref{theo:new_bound_general} and~\ref{theo:catoni_new_general} to the risk of a linear classifier $h_{\wbf}$, defined by a weight vector~${\wbf}\in\Rbb^d$\,:\\[2mm]
\centerline{$
\forall \xbf\in\X,\ \ 
h_{\wbf}(\xbf) \, =\,   \sgn\left(\wbf\cdot\xbf\right).$}\\[2mm]
%
The taken approach is the one privileged in numerous PAC-Bayesian works \citep[\eg,][]{Langford02,AmbroladzePS06,mcallester-keshet-11,Parrado-Hernandez12,germain2009pac,pbda}, as it makes the risk of the linear classifier $h_\wb$ and the risk of a (properly parametrized) majority vote coincide, while in the same time promoting large margin classifiers. 
To this end, let $\Hcal$ be the set of \emph{all} linear classifiers over the input space,
$\Hcal \! =\!  \left\{ h_{\wbf'} \mid \wbf'\in\Rbb^d  \right\},$
and let $\posterior_\wb$ over $\Hcal$ be a \emph{posterior} distribution,  {\it resp.} a prior distribution $\prior_{\bf 0}$, 
that is constrained to be a spherical Gaussian with identity covariance matrix centered on vector $\wb$, {\it resp.} ${\bf 0}$, 
\begin{align*}
	\forall h_{\wb'}\in \Hcal,\quad   \posterior_\wb(h_{\wbf'})  \, &=\,   \left( \tfrac{1}{\sqrt{2\pii}} \right)^{ d}
	e^{-\tfrac{1}{2}\|{\wbf'}-\wb\|^2},
	\\
	\text{and}  \quad \prior_{\bf 0}(h_{\wbf'}) \,&=\, \left( \tfrac{1}{\sqrt{2\pii}} \right)^{ d}
	e^{-\tfrac{1}{2}\|{\wbf'}\|^2}.
\end{align*}
The KL-divergence between  $\posterior_\wb$ and  $\pi_{\bf 0}$ simply is  
\begin{equation}
\label{eq:KL_linear}
\KL(\posterior_\wb\|\prior_{\bf 0})   =  \tfrac{1}{2}\|\wb\|^2\,.
\end{equation}
Thanks to this parameterization, 
the majority vote classifier~$B_{\posterior_\wb}$ corresponds to the one of the linear classifier $h_\wb$ (see above cited PAC-Bayesian works). That is,\\[1mm] 
\centerline{$\displaystyle
	\forall\,\xbf{\in} X, \wb{\in} \Hcal, \ 
	h_\wb(\xbf) 
	=  \sgn\left[ \esp{h_{\wb'} \sim \posterior_\wb} \!\!\!\!h_{\wb'}(\xbf) \right]\!
	 = \!B_{\posterior_\wb}(\xbf)\,.$}\\[1mm]
Then, 
$\Risk_\P (h_\wb) =  \Risk_\P (B_{\posterior_\wb})$
for any data distribution~$\P$. \\
Moreover, 
\citet{Langford02} showed that the closely related Gibbs risk (Equation~\ref{eq:gibbs_risk}) is
related to the linear classifier margin $y\, \frac{\wb \cdot \xbf}{\|\xbf\|}$, as follows:
\begin{equation} \label{eq:gibbs_risk_linear}
	\Risk_\P (G_{\posterior_\wb})
	\ = \ \Esp_{(\xbf,y)\sim \P} \Phi \left(  y\, \frac{\wb \cdot \xbf}{\|\xbf\|}  \right),
\end{equation}
where 
$\Phi(x)   {=}  \tfrac12 {-} \frac12\Erf\big( \frac{x}{\sqrt{2}} \big),$ and 
\mbox{$\Erf(x)  =  \tfrac{2}{\sqrt{\pii}} \int_0^x e^{-t^2} \text{d}t$}  
is the Gauss error function. 
Here, $\Phi(x)$ can be seen as a \emph{smooth} surrogate---sometimes called the \emph{probit loss} \citep[\eg,][]{mcallester-keshet-11}---of the zero-one loss function $\I{x\leq0}$ relying on $y\, \frac{\wb \cdot \xbf}{\|\xbf\|}$. Note that $\|\wb\|$ plays an important role on the value of $\Risk_\P (G_{\posterior_\wb})$, but not on $\Risk_\P (h_\wb)$. Indeed, $\Risk_\P (G_{\posterior_\wb})$ tends to $\Risk_\P (h_\wb)$ as $\|\wb\|$ grows, which can provide \emph{very tight bounds} \citep[see the empirical analyses of][]{AmbroladzePS06,germain2009pac}.
In the PAC-Bayesian context, $\|\wb\|$ turns out to be a measure of \emph{complexity} of the learned classifier, as Equation~\eqref{eq:KL_linear} shows.
\\
We now seek to express the \emph{expected disagreement} $\dD(\posterior_\wb)$ and the \emph{expected joint error}  $\eP(\posterior_\wb)$ of Equations~\eqref{eq:voters_dis} and~\eqref{eq:erreur_jointe} related to the parameterized distribution $\posterior_\wb$.  As shown in \citet{pbda} the former is given by\\[1mm]
\centerline{$\displaystyle
	\dD(\posterior_\wb)
	=  \Esp_{\tbf\sim \D} \Phidis  \left(  \frac{\wb \cdot \xbf}{\|\xbf\|}  \right), $}\\[1mm]
where $\Phidis(x) = 2{\times}\Phi(x){\times}\Phi(-x)$. 
Following a similar approach, we obtain,
for all $\wb\in\Rbb$,
\begin{eqnarray*}
	\ep(\posterior_\wb)
	&=&
	\Esp_{\exbf\sim\P} \Esp_{h\sim\posterior_\wb} \Esp_{h'\sim\posterior_\wb}  \I{h(\xbf){\ne }y}\I{h'(\xbf){\ne} y} \\
	&=&
	\Esp_{\exbf\sim\P}  \Esp_{h\sim\posterior_\wb}   \I{h(\xbf){\ne} y}
	\Esp_{h'\sim\posterior_\wb} \I{h'(\xbf){\ne} y}\\[-1.5mm]
	&=&
	\Esp_{\exbf\sim\P}  
	\Phierr \left(  y\, \frac{\wb \cdot \xbf}{\|\xbf\|}  \right),
\end{eqnarray*}
with  $\Phierr(x) \!=\! \big[\Phi(x)\big]^2$.
As function $\Phi$ in Equation~\eqref{eq:gibbs_risk_linear}, functions $\Phierr$ and $\Phidis$ defined above can be interpreted as loss functions for linear classifiers (illustrated by Figure~\ref{fig:losses}).
\\[0.5mm]	
\textbf{Domain adaptation bound.}
Theorem~\ref{theo:new_bound_general} specialized to linear classifiers gives the following corollary. Note that, as mentioned above, $\RPT(h_\wb) = \RPT(B_{\posterior_\wb})  \leq 2\,\RPT(G_{\posterior_\wb})$. 
\begin{cor}\label{cor:new_bound_linear}
 Let $\PS$ and $\PT$ respectively be the source and the target domains on $\XY$. For all $\wb\in \Rbb$, we have :\\[1mm]
\centerline{$\displaystyle \RPT(h_\wb) \, \leq \, \dDT(\Q_\wb) +
			2\,\binf \times  \ePS(\Q_\wb)
			+ 2\,\ets  \,,$}
\end{cor}
Figure~\ref{fig:losses} leads to an insightful geometric interpretation of the domain adaptation trade-off
promoted by Corollary~\ref{cor:new_bound_linear}.
For fixed values of $\binf$ and $\ets$, the target risk $\RPT(h_\wb)$ is upper-bounded by a ($\beta_\infty$-weighted) sum of two losses.
 The  expected $\Phierr$-loss (\ie, the joint error) is computed on the (labeled) source domain; it aims to label the source examples correctly, but is more permissive on the required margin than the $\Phi$-loss (\ie, the Gibbs risk). The expected $\Phidis$-loss (\ie, the disagreement) is computed on the target (unlabeled) domain; it promotes large \emph{unsigned} target margins.
Thus, if a target domain fulfills the \emph{cluster assumption} (described in Section~\ref{section:assumptions}),  $\dDT(\posterior_\wb)$ will be low when the decision boundary crosses a low-density region between the homogeneous labeled clusters. 
Hence, Corollary~\ref{cor:new_bound_linear} reflects that some source errors may be allowed if, doing so, the separation of the target domain is improved.
\\[0.5mm]
\textbf{Generalization bound and learning algorithm.}
Theorem~\ref{theo:catoni_new_general} specialized to linear classifiers gives the following.
	\begin{cor}
		\label{cor:gen_bound linear}
		For any domains $\PS$ and $\PT$ over $\X{\times}\Y$, any $\delta{\in}(0,1]$, any 
		$a {>} 0$ and $b{>}0$, with a probability at least $1{-}\delta$ over the choices of $S{\sim}(\PS)^\ms$ and $ T {\sim} (\DT)^\mt$, we have \\[-4mm]
		\begin{align*}
		\forall \wb\in\Rbb: \,
		\RPT(h_\wb) & \leq\, c'\,\dT(\posterior_\wb)  +2\,b'\,
		\eS(\posterior_\wb)  + 2\,\ets \\
		& {}+ 2\left(\tfrac{c'}{\mt\times c} + \tfrac{b'}{\ms\times b} \right)   \Big(\|\wb\|^2 + \ln \tfrac{2}{\delta}\Big).
		\end{align*}
	\end{cor}
	For a source $S  {=} \{(\xb_i, y_i)\}_{i=1}^{\ms}$ and a target $T  {=}  \{(\xbf_i')\}_{i=1}^{\mt}$ samples of potentially \emph{different size},  and some hyperparameters $C{>}0$, $B{>}0$, 
	minimizing the next objective function w.r.t $\wb{\in}\Rbb$ is equivalent to minimize the above bound.
	\begin{align} \label{eq:ca}
		C&\,\dT(\posterior_\wb) + B\,\eS(\posterior_\wb) 	+ 
		 \|\wb\|^2 
		\\ \nonumber
		&= \  C \sum_{i=1}^{\mt} \Phidis \left(  \tfrac{\wb \cdot \tbfi'}{\|\tbfi'\|}  \right)  
\!+ B\sum_{i=1}^\ms \Phierr \left(  y_i\, \tfrac{\wb \cdot \xbf_i}{\|\xbf_i\|}  \right)
		+\|\wb\|^2\,.
	\end{align}
	We call the optimization of Equation~\eqref{eq:ca} by gradient descent the \algo algorithm, for Domain Adaptation of Linear Classifiers.
	The kernel trick applies to \algo. 
	That is, given a kernel $k\!:\!\R^d {\times} \R^d{\rightarrow}\R$, one can express a linear classifier in a \emph{RKHS}\footnote{It is non-trivial to show that the kernel trick holds when $\prior_\mathbf{0}$ and $\posterior_\wb$ are Gaussian over infinite-dimensional feature space. As mentioned by \citet{mcallester-keshet-11}, it is, however, the case provided we consider Gaussian processes as measure of distributions $\prior_\mathbf{0}$ and $\posterior_\wb$ over (infinite) $\Hcal$.} by a dual weight vector $\alphab\in \R^{\ms+\mt}$\,:\\[1mm]
\centerline{$\displaystyle	h_\wb(\cdot) =\sgn\left[\sum_{i=1}^\ms \alpha_i k(\xb_i, \cdot) +  \sum_{i=1}^\mt \alpha_{i+\ms} k(\xbf_i', \cdot)\right].$}\\[1mm]
	Even though the objective function is highly non-convex, we achieved good empirical results by minimizing the ``kernelized'' version of Equation~\eqref{eq:ca} by gradient descent, with a uniform weight vector as a starting point. More details are given in the supplementary material.

		\begin{table}[t]
			\centering
			\caption{Error rates on \emph{Amazon} dataset. 
				Best risks appear in {\bf bold} and seconds are in {\it italic}.
				\label{tab:res_sentiments}}
			{ \scriptsize
				\begin{tabular}{|l||c|c@{\hspace{3mm}}c@{\hspace{3mm}}c|c|}
					\toprule
					& {\sc svm} & {\sc dasvm} & {\sc coda} & {\sc pbda} & 
					\algo \\ 
					& {\tiny (CV)}  & {\tiny(RCV)} & {\tiny(RCV)}  &  {\tiny(RCV)}   & 
					 {\tiny(RCV)} \\ 
					\midrule
					{\small books$\rightarrow$DVDs} 
					& ${\it 0.179}$ & $0.193$ & ${0.181}$ & $0.183$ & ${\bf 0.178}$ \\ 
					{\small books$\rightarrow$electro }
					& $0.290$ & ${\it 0.226}$ & ${0.232}$ & $0.263$ & ${\bf 0.212}$\\ 
					{\small books$\rightarrow$kitchen }
					& $0.251$ & ${\bf 0.179}$ & ${ 0.215}$ & $0.229$ & ${\it 0.194}$\\ 
					{\small DVDs$\rightarrow$books  }
					& $0.203$ & $0.202$ & $0.217$ & ${\it 0.197}$ & ${\bf 0.186}$\\ 
					{\small DVDs$\rightarrow$electro  }
					& $0.269$ & ${\bf 0.186}$ & ${\it 0.214}$ & $0.241$ & $0.245$\\ 
					{\small DVDs$\rightarrow$kitchen  }
					& $0.232$ & $0.183$ & ${\it 0.181}$ & $0.186$ & ${\bf 0.175}$\\ 
					{\small electro$\rightarrow$books }
					& $0.287$ & $0.305$ & $0.275$ & ${\bf 0.232}$ & ${\it 0.240}$\\ 
					{\small electro$\rightarrow$DVDs }
					& $0.267$ & ${\bf 0.214}$ & $0.239$ & ${\it 0.221}$ & $0.256$\\ 
					{\small electro$\rightarrow$kitchen  }
					& ${\it 0.129}$ & $0.149$ & ${0.134}$ & $0.141$ & ${\bf 0.123}$\\ 
					{\small kitchen$\rightarrow$books  }
					& $0.267$ & $ 0.259$ & ${\it 0.247}$ & ${\it 0.247}$ & ${\bf 0.236}$\\ 
					{\small kitchen$\rightarrow$DVDs  }
					& $0.253$ & ${\bf 0.198}$ & $0.238$ & ${0.233}$ & ${\it 0.225}$\\ 
					{\small kitchen$\rightarrow$electro  }
					& $0.149$ & $0.157 $ & $0.153$ & ${\bf 0.129}$ & ${\it 0.131}$\\ 
					\midrule
					{\small Average  }
					& $0.231 $ & ${\it 0.204} $ & $0.210 $ & ${0.208} $ & ${\bf 0.200} $ \\ 
					\bottomrule
				\end{tabular}
			}
		\end{table}
		
\section{Experimental Results}
\label{sec:expe}

Firstly, Figure~\ref{fig:lunes} illustrates the behavior
of the decision boundary of our algorithm \algo on an intertwining moons toy problem\footnote{We generate each pair of moons with the {\tt make\_moons} function provided in {\tt scikit-learn}~\cite{scikit-learn}.}, where each moon corresponds to a label. The target domain, for which we have no label, is a rotation of the source one. The figure shows clearly that \algo succeeds to adapt to the target domain, even for a rotation angle of $50\degree$.
We see that \algo does not rely on the restrictive \emph{covariate shift} assumption, as some source examples are misclassified.
This behavior illustrates the \algo trade-off in action, that concedes some errors on the source sample to lower the disagreement on the target sample. 
 	
	Secondly, we evaluate \algo
	on the classical {\it Amazon.com Reviews} benchmark~\cite{BlitzerMP06} according to the setting used by \citet{coda,pbda}. This dataset contains reviews of four types of products (books, DVDs, electronics, and kitchen appliances) described with about $100,000$ attributes.
	Originally, the reviews were labeled with a rating from $1$ to $5$.
	\citet{coda} proposed a simplified binary setting by regrouping ratings into two classes (products rated lower than $3$ and products rated higher than $4$). Moreover, they reduced the dimensionality to about $40,\!000$ by only keeping the features appearing at least ten times for a given domain adaptation task. 
	Finally, the data are pre-processed with a tf-idf re-weighting. 
	A domain corresponds to a kind of product. Therefore, we perform twelve domain adaptation tasks. 
	For instance, ``books$\to$DVD's'' is the task for which the source domain is ``books'' and the target one is ``DVDs''. 
	We compare \algo with the classical non-adaptive algorithm {\sc svm} (trained only on the source sample),
	the adaptive algorithm {\sc dasvm}~\cite{dasvm},  the adaptive co-training {\sc coda}~\cite{coda}, and the PAC-Bayesian domain adaptation algorithm {\sc pbda}~\cite{pbda} based on Theorem~\ref{theo:pbda}. 
	Note that, in \citet{pbda}, {\sc dasvm}
	has shown better accuracy than {\sc svm}, {\sc coda} and {\sc pbda}.
	Each parameter is selected with a grid search thanks to a usual cross-validation (CV) on the source sample for {\sc svm},  
	and thanks to a reverse validation procedure\footnote{For details on the reverse validation procedure, see \citet{dasvm,Zhong-ECML10}. Other details on our experimental protocol are given in supplementary material.} (RCV) for {\sc coda}, {\sc dasvm}, {\sc pbda}, and \algo.
	The algorithms use a linear kernel and consider $2,\!000$ labeled source examples and $2,\!000$ unlabeled target examples. 
	Table~\ref{tab:res_sentiments} reports the error rates of all the methods evaluated on the same separate target test sets proposed by \citet{coda}.

	Above all, 
        the adaptive approaches show the best result, implying that tackling this problem with a domain adaptation method is reasonable. Then, our new method \algo~is the best algorithm overall on this task. 
	Except for the two adaptive tasks between ``electronics'' and ``DVDs'', \algo~is either the best one (six times), or the second one (four times). 
	Moreover, according to a Wilcoxon signed rank test with a $5\%$ significance level, we obtain a probability of $89.5\%$ that {\sc DALC} is better than {\sc PBDA}. 
This test tends to confirm that our new bound improves the analysis done previously in \citet{pbda}, in addition to being more interpretable.

	\section{Conclusion}
	\label{sec:conclu}
	\label{sec:conclusion}
	We propose a new domain adaptation analysis for majority vote learning.
	It relies on an upper bound on the target risk, expressed 
        as a trade-off between the voters' disagreement on the target domain, the voters' joint errors on the source one, and a term 
        reflecting the worst case error
in regions where the source domain is non-informative. To the best of our knowledge, a crucial novelty of our contribution is that the trade-off is controlled by the divergence $\bq$ (Equation~\ref{eq:bq}) between the domains: The divergence is not an additive term (as in many domain adaptation bounds) but is a factor weighting the importance of the source information.
		Our analysis, combined with a PAC-Bayesian generalization bound, leads to a new domain adaptation algorithm for linear classifiers. 
The empirical experiments show that our new algorithm outperforms the previous PAC-Bayesian approach \citep{pbda}.

As future work, we first aim at investigating the case where the domains' divergence $\bq$ can be estimated, \ie, when the covariate shift assumption holds or when some target labels are available. In these scenarios, $\bq$ might not be considered as a hyperparameter to tune.
\\
%
	Last but not least, the term $\ets$ of our bound---suggesting that the two domains should live in the same regions---can be dealt with a representation learning approach.
		As mentioned in Section~\ref{section:representation_learning}, this could be an incentive to combine our learning algorithm with existing representation learning techniques.  
In another vein, considering an \emph{active learning} setup \citep[as in][]{berlind-15}, one could query the labels of target examples to estimate the value bounded by $\ets$.
We see this as a great source of inspiration for new algorithms for this learning paradigm.



	\section*{Acknowledgements}
This work was supported in
part by the French project LIVES ANR-15-CE23-0026-03, and in part
by NSERC discovery grant 262067. 
	
	\bibliography{biblio}

\begin{thebibliography}{46}
\providecommand{\natexlab}[1]{#1}
\providecommand{\url}[1]{\texttt{#1}}
\expandafter\ifx\csname urlstyle\endcsname\relax
  \providecommand{\doi}[1]{doi: #1}\else
  \providecommand{\doi}{doi: \begingroup \urlstyle{rm}\Url}\fi

\bibitem[Ambroladze et~al.(2006)Ambroladze, Parrado-Hern{\'a}ndez, and
  Shawe-Taylor]{AmbroladzePS06}
Ambroladze, A., Parrado-Hern{\'a}ndez, E., and Shawe-Taylor, J.
\newblock Tighter {PAC-Bayes} bounds.
\newblock In \emph{NIPS}, pp.\  9--16, 2006.

\bibitem[Ben{-}David et~al.(2006)Ben{-}David, Blitzer, Crammer, and
  Pereira]{BDnips}
Ben{-}David, S., Blitzer, J., Crammer, K., and Pereira, F.
\newblock Analysis of representations for domain adaptation.
\newblock In \emph{NIPS}, pp.\  137--144, 2006.

\bibitem[Ben{-}David et~al.(2010)Ben{-}David, Blitzer, Crammer, Kulesza,
  Pereira, and Vaughan]{BDjournal}
Ben{-}David, S., Blitzer, J., Crammer, K., Kulesza, A., Pereira, F., and
  Vaughan, J.~Wortman.
\newblock A theory of learning from different domains.
\newblock \emph{Mach. Learn.}, 79\penalty0 (1-2):\penalty0 151--175, 2010.

\bibitem[Ben{-}David et~al.(2012)Ben{-}David, Shalev{-}Shwartz, and
  Urner]{bendavid-12}
Ben{-}David, S., Shalev{-}Shwartz, S., and Urner, R.
\newblock Domain adaptation--can quantity compensate for quality?
\newblock In \emph{ISAIM}, 2012.

\bibitem[Berlind \& Urner(2015)Berlind and Urner]{berlind-15}
Berlind, C. and Urner, R.
\newblock Active nearest neighbors in changing environments.
\newblock In \emph{ICML}, pp.\  1870--1879, 2015.

\bibitem[Blitzer(2007)]{blitzer2007domain}
Blitzer, J.
\newblock \emph{Domain adaptation of natural language processing systems}.
\newblock PhD thesis, UPenn, 2007.

\bibitem[Blitzer et~al.(2006)Blitzer, McDonald, and Pereira]{BlitzerMP06}
Blitzer, J., McDonald, R., and Pereira, F.
\newblock Domain adaptation with structural correspondence learning.
\newblock In \emph{EMNLP}, pp.\  120--128, 2006.

\bibitem[Bruzzone \& Marconcini(2010)Bruzzone and Marconcini]{dasvm}
Bruzzone, L. and Marconcini, M.
\newblock Domain adaptation problems: A {DASVM} classification technique and a
  circular validation strategy.
\newblock \emph{IEEE Trans. Pattern Anal. Mach. Intel.}, 32\penalty0
  (5):\penalty0 770--787, 2010.

\bibitem[Catoni(2007)]{catoni2007pac}
Catoni, O.
\newblock \emph{{PAC-Ba}yesian supervised classification: the thermodynamics of
  statistical learning}, volume~56.
\newblock Inst. of Mathematical Statistic, 2007.

\bibitem[Chen et~al.(2011)Chen, Weinberger, and Blitzer]{coda}
Chen, M., Weinberger, K.Q., and Blitzer, J.
\newblock Co-training for domain adaptation.
\newblock In \emph{NIPS}, pp.\  2456--2464, 2011.

\bibitem[Chen et~al.(2012)Chen, Xu, Weinberger, and Sha]{chen-12}
Chen, M., Xu, Z.~E., Weinberger, K.~Q., and Sha, F.
\newblock Marginalized denoising autoencoders for domain adaptation.
\newblock In \emph{ICML}, pp.\  767--774, 2012.

\bibitem[Cortes \& Mohri(2014)Cortes and Mohri]{CortesM14}
Cortes, C. and Mohri, M.
\newblock Domain adaptation and sample bias correction theory and algorithm for
  regression.
\newblock \emph{Theor. Comput. Sci.}, 519:\penalty0 103--126, 2014.

\bibitem[Cortes et~al.(2010)Cortes, Mansour, and Mohri]{CortesMM10}
Cortes, C., Mansour, Y., and Mohri, M.
\newblock Learning bounds for importance weighting.
\newblock In \emph{NIPS}, pp.\  442--450, 2010.

\bibitem[Cortes et~al.(2015)Cortes, Mohri, and Medina]{cortes-15}
Cortes, C., Mohri, M., and Medina, A.~Mu{\~{n}}oz.
\newblock Adaptation algorithm and theory based on generalized discrepancy.
\newblock In \emph{{ACM} {SIGKDD}}, pp.\  169--178, 2015.

\bibitem[{Daum\'e III}(2007)]{daume07easyadapt}
{Daum\'e III}, H.
\newblock Frustratingly easy domain adaptation.
\newblock In \emph{ACL}, 2007.

\bibitem[Ganin \& Lempitsky(2015)Ganin and Lempitsky]{ganin-15}
Ganin, Y. and Lempitsky, V.~S.
\newblock Unsupervised domain adaptation by backpropagation.
\newblock In \emph{ICML}, pp.\  1180--1189, 2015.

\bibitem[Ganin et~al.(2016)Ganin, Ustinova, H, Germain, Larochelle, Laviolette,
  Marchand, and Lempitsky]{ganin-16}
Ganin, Y., Ustinova, E., H, Ajakan, Germain, P., Larochelle, H., Laviolette,
  F., Marchand, M., and Lempitsky, V.
\newblock Domain-adversarial training of neural networks.
\newblock \emph{JMLR}, 17\penalty0 (59):\penalty0 1--35, 2016.

\bibitem[Germain et~al.(2009)Germain, Lacasse, Laviolette, and
  Marchand]{germain2009pac}
Germain, P., Lacasse, A., Laviolette, F., and Marchand, M.
\newblock {PAC-Ba}yesian learning of linear classifiers.
\newblock In \emph{ICML}, pp.\  353--360, 2009.

\bibitem[Germain et~al.(2013)Germain, Habrard, Laviolette, and Morvant]{pbda}
Germain, P., Habrard, A., Laviolette, F., and Morvant, E.
\newblock A {PAC-B}ayesian approach for domain adaptation with specialization
  to linear classifiers.
\newblock In \emph{ICML}, pp.\  738--746, 2013.

\bibitem[Germain et~al.(2015)Germain, Lacasse, Laviolette, Marchand, and
  Roy]{graal-neverending}
Germain, P., Lacasse, A., Laviolette, F., Marchand, ., and Roy, J.-F.
\newblock Risk bounds for the majority vote: From a {PAC-Bayesian} analysis to
  a learning algorithm.
\newblock \emph{JMLR}, 16:\penalty0 787--860, 2015.

\bibitem[Herbrich \& Graepel(2000)Herbrich and Graepel]{herbrich-00}
Herbrich, R. and Graepel, T.
\newblock A {PAC-B}ayesian margin bound for linear classifiers: Why svms work.
\newblock In \emph{NIPS}, pp.\  224--230, 2000.

\bibitem[Huang et~al.(2006)Huang, Smola, Gretton, Borgwardt, and
  Sch{\"o}lkopf]{HuangSGBS-nips06}
Huang, J., Smola, A., Gretton, A., Borgwardt, K., and Sch{\"o}lkopf, B.
\newblock Correcting sample selection bias by unlabeled data.
\newblock In \emph{NIPS}, pp.\  601--608, 2006.

\bibitem[Jiang(2008)]{jiang2008literature}
Jiang, J.
\newblock A literature survey on domain adaptation of statistical classifiers,
  2008.

\bibitem[Jones et~al.(2001--)Jones, Oliphant, Peterson, et~al.]{scipy}
Jones, E., Oliphant, T., Peterson, P., et~al.
\newblock {SciPy}: Open source scientific tools for {Python}, 2001--.
\newblock URL \url{http://www.scipy.org/}.

\bibitem[Lacasse et~al.(2006)Lacasse, Laviolette, Marchand, Germain, and
  Usunier]{graal-nips06-mv}
Lacasse, A., Laviolette, F., Marchand, M., Germain, P., and Usunier, N.
\newblock {PAC-B}ayes bounds for the risk of the majority vote and the variance
  of the {G}ibbs classifier.
\newblock In \emph{NIPS}, pp.\  769--776, 2006.

\bibitem[Langford \& Shawe-Taylor(2002)Langford and Shawe-Taylor]{Langford02}
Langford, J. and Shawe-Taylor, J.
\newblock {PAC-Bayes} \& margins.
\newblock In \emph{NIPS}, pp.\  439--446, 2002.

\bibitem[Li \& Bilmes(2007)Li and Bilmes]{LiB07}
Li, X. and Bilmes, J.
\newblock A {B}ayesian divergence prior for classiffier adaptation.
\newblock In \emph{{AISTATS}}, pp.\  275--282, 2007.

\bibitem[Liu et~al.(2008)Liu, Mackey, Roos, and Pereira]{liu2008evigan}
Liu, Q., Mackey, A.~J., Roos, D.~S., and Pereira, F.
\newblock Evigan: a hidden variable model for integrating gene evidence for
  eukaryotic gene prediction.
\newblock \emph{Bioinformatics}, 24\penalty0 (5):\penalty0 597--605, 2008.

\bibitem[Mansour et~al.(2009)Mansour, Mohri, and Rostamizadeh]{Mohri}
Mansour, Y., Mohri, M., and Rostamizadeh, A.
\newblock Domain adaptation: Learning bounds and algorithms.
\newblock In \emph{COLT}, 2009.

\bibitem[Margolis(2011)]{margolis2011literature}
Margolis, A.
\newblock A literature review of domain adaptation with unlabeled data, 2011.

\bibitem[Maurer(2004)]{maurer-04}
Maurer, A.
\newblock A note on the {PAC}-{B}ayesian theorem.
\newblock \emph{CoRR}, cs.LG/0411099, 2004.

\bibitem[McAllester(2013)]{mcallester-13}
McAllester, D.
\newblock A {PAC-B}ayesian tutorial with a dropout bound.
\newblock \emph{CoRR}, abs/1307.2118, 2013.

\bibitem[McAllester(1999)]{Mcallester99a}
McAllester, D.~A.
\newblock Some {PAC-Bayesian} theorems.
\newblock \emph{Mach. Learn.}, 37:\penalty0 355--363, 1999.

\bibitem[McAllester \& Keshet(2011)McAllester and Keshet]{mcallester-keshet-11}
McAllester, D.~A. and Keshet, J.
\newblock Generalization bounds and consistency for latent structural probit
  and ramp loss.
\newblock In \emph{NIPS}, pp.\  2205--2212, 2011.

\bibitem[Morvant et~al.(2012)Morvant, Habrard, and Ayache]{morvant12}
Morvant, E., Habrard, A., and Ayache, S.
\newblock {Parsimonious Unsupervised and Semi-Supervised Domain Adaptation with
  Good Similarity Functions}.
\newblock \emph{{KAIS}}, 33\penalty0 (2):\penalty0 309--349, 2012.

\bibitem[Pan \& Yang(2010)Pan and Yang]{pan2010survey}
Pan, S.~J. and Yang, Q.
\newblock A survey on transfer learning.
\newblock \emph{T. Knowl. Data En.}, 22\penalty0 (10):\penalty0 1345--1359,
  2010.

\bibitem[Parrado-Hern{\'a}ndez et~al.(2012)Parrado-Hern{\'a}ndez, Ambroladze,
  Shawe-Taylor, and Sun]{Parrado-Hernandez12}
Parrado-Hern{\'a}ndez, E., Ambroladze, A., Shawe-Taylor, J., and Sun, S.
\newblock {PAC-Bayes} bounds with data dependent priors.
\newblock \emph{JMLR}, 13:\penalty0 3507--3531, 2012.

\bibitem[Patel et~al.(2015)Patel, Gopalan, Li, and Chellappa]{patel2015visual}
Patel, V.~M., Gopalan, R., Li, R., and Chellappa, R.
\newblock Visual domain adaptation: A survey of recent advances.
\newblock \emph{IEEE Signal Proc. Mag.}, 32\penalty0 (3):\penalty0 53--69,
  2015.

\bibitem[Pedregosa et~al.(2011)Pedregosa, Varoquaux, Gramfort, Michel, Thirion,
  Grisel, Blondel, Prettenhofer, Weiss, Dubourg, Vanderplas, Passos,
  Cournapeau, Brucher, Perrot, and Duchesnay]{scikit-learn}
Pedregosa, F., Varoquaux, G., Gramfort, A., Michel, V., Thirion, B., Grisel,
  O., Blondel, M., Prettenhofer, P., Weiss, R., Dubourg, V., Vanderplas, J.,
  Passos, A., Cournapeau, D., Brucher, M., Perrot, M., and Duchesnay, E.
\newblock Scikit-learn: Machine learning in {P}ython.
\newblock \emph{JMLR}, 12:\penalty0 2825--2830, 2011.

\bibitem[Quionero-Candela et~al.(2009)Quionero-Candela, Sugiyama, Schwaighofer,
  and Lawrence]{quionero2009dataset}
Quionero-Candela, J., Sugiyama, M., Schwaighofer, A., and Lawrence, N.~D.
\newblock \emph{Dataset shift in machine learning}.
\newblock The MIT Press, 2009.

\bibitem[Seldin \& Tishby(2010)Seldin and Tishby]{seldin-10}
Seldin, Y. and Tishby, N.
\newblock {PAC-B}ayesian analysis of co-clustering and beyond.
\newblock \emph{JMLR}, 11:\penalty0 3595--3646, 2010.

\bibitem[Shimodaira(2000)]{covariateshift}
Shimodaira, H.
\newblock Improving predictive inference under covariate shift by weighting the
  log-likelihood function.
\newblock \emph{J. Statist. Plann. Inference}, 90\penalty0 (2):\penalty0
  227--244, 2000.

\bibitem[Sugiyama et~al.(2007)Sugiyama, Nakajima, Kashima, von B{\"u}nau, and
  Kawanabe]{Sugiyama-NIPS07}
Sugiyama, M., Nakajima, S., Kashima, H., von B{\"u}nau, P., and Kawanabe, M.
\newblock Direct importance estimation with model selection and its application
  to covariate shift adaptation.
\newblock In \emph{NIPS}, 2007.

\bibitem[Urner et~al.(2011)Urner, Shalev{-}Shwartz, and Ben{-}David]{urner-11}
Urner, R., Shalev{-}Shwartz, S., and Ben{-}David, S.
\newblock Access to unlabeled data can speed up prediction time.
\newblock In \emph{ICML}, pp.\  641--648, 2011.

\bibitem[Zhang et~al.(2012)Zhang, Zhang, and Ye]{zhang2012generalization}
Zhang, C., Zhang, L., and Ye, J.
\newblock Generalization bounds for domain adaptation.
\newblock In \emph{NIPS}, pp.\  3320--3328, 2012.

\bibitem[Zhong et~al.(2010)Zhong, Fan, Yang, Verscheure, and Ren]{Zhong-ECML10}
Zhong, E., Fan, W., Yang, Q., Verscheure, O., and Ren, J.
\newblock Cross validation framework to choose amongst models and datasets for
  transfer learning.
\newblock In \emph{ECML-PKDD}, pp.\  547--562, 2010.

\end{thebibliography}
	\bibliographystyle{icml2016}

\newpage 

\appendix
\section*{Supplementary Material}

\section{Proof of Theorem~4}
\newcommand{\LD}{\mathcal{L}_{\P}}
\newcommand{\LS}{\mathcal{L}_{S}}
\newcommand{\LSp}{\mathcal{L}_{S'}}

\begin{proof}
We use the following shorthand notation:
$$\LD(h) = \esp{\exbf\sim\P}  \ell(h, \xb,y)$$
\mbox{ and }
$$\LS(h) = \frac1m\sum_{\exbf\in S}  \ell(h, \xb,y)\,.
$$
Consider any convex function $\Delta:[0,1]{\times}[0,1]\to\R$.
Applying consecutively Jensen's Inequality and the \emph{change of measure inequality} (see \citet[Lemma 4]{seldin-10} and \citet[Equation (20)]{mcallester-13}), we obtain
\begin{align*}
\forall \posterior \mbox{ on } \Hcal\,: \quad  & m{\times}\Delta\left(
\esp{h\sim\posterior} \LS(h),  \esp{h\sim\posterior} \LD(h)
\right) \\
&\ \leq \ 
\esp{h\sim\posterior}
m {\times} \Delta\left(
\LS(h), \LD(h)
\right)\\
&\ \leq \ \KL(\posterior\|\prior)
+ \ln \Big[
X_\prior(S)
\Big]\,,
\end{align*}
with
\begin{equation*}
X_\prior(S) = 
\esp{h\sim\prior}
e^{m {\times} \Delta\left(
 \LS(h),\, \LD(h)
\right)}.
\end{equation*}
Then, Markov's Inequality gives
\begin{equation*}
\Pr_{S\sim\P^m} \left(
X_\prior (S) \leq \tfrac1\delta \esp{S'\sim\P^m} X_\prior (S')
\right)
\ \geq \ 
1{-}\delta\,, 
\end{equation*}
and 
\begin{align}
\nonumber
&\esp{S'\sim\P^m} X_\prior (S')
\ =\ 
\esp{S'\sim\P^m} \esp{h\sim\prior}
e^{m {\times} \Delta\left(
 \LSp(h),\, \LD(h)
\right)}
 \\
\nonumber
 =& \ 
\esp{h\sim\prior} \esp{S'\sim\P^m} 
e^{m {\times} \Delta\left(
 \LSp(h),\, \LD(h)
\right)}
\\
\label{eq:aaa}
\leq & \ \esp{h\sim\prior} 
\sum_{k=0}^m
\binom{k}m (\LD(h))^k (1{-}\LD(h))^{m-k}
e^{m {\times} \Delta\left(
 \frac{k}m,\, \LD(h)
\right)},
\end{align}
where the last inequality is due to \citet[Lemma 3]{maurer-04} (we have an equality when the output of $\ell$ is in $\{0,1\}$).
As shown in \citet[Corollary 2.2]{germain2009pac}, by fixing
 $$\Delta(q,p) = -c{\times} q -\ln[1{-}p\,(1{-}e^{-c})]\,,$$
 Line~\ref{eq:aaa} becomes equal to $1$, and then $\esp{S'\sim\P^m} X_\prior (S')\leq 1$.
Hence,
\begin{align*}
\Pr_{S\sim\P^m} \Bigg(
\forall \posterior \mbox{ on } \Hcal\,:  \,
-c \esp{h\sim\posterior} \LS(h) -\ln[1{-}\esp{h\sim\posterior} \LD(h)\,(1{-}e^{-c})]&\\
 \leq  
 \frac{\KL(\posterior\|\prior)
 + \ln  \tfrac1\delta}{m} 
\Bigg)
\ \geq \ 
1{-}\delta&\,.
\end{align*}
By reorganizing the terms, we have, with probability $1{-}\delta$ over the choice of $S\in \P^m$, 
\begin{align*}
&\forall \posterior \mbox{ on } \Hcal\,: \,
 \esp{h\sim\posterior} \LD(h)\\
 &\leq  
\frac1{1{-}e^{-c}}\left[1-
\exp\left(-c \esp{h\sim\posterior} \LS(h) -
 \frac{\KL(\posterior\|\prior)
 + \ln  \tfrac1\delta}{m} \right)\right].
\end{align*}
The final result is obtained by using the inequality $1{-}\exp(-z) \leq z$. 
\end{proof}

\section{Using \algo with a kernel function}

\newcommand{\m}{\textsc{m}}
Let
$S\!=\!\{(\xbf_i,y_i)\}_{i=1}^{\ms}$,\,\, $T\!=\!\{\xbf_i'\}_{i=1}^{\mt}$
\,and\, $\m=\ms+\mt$. 
We will denote
$$\xb_\# \, =\, 
\begin{cases}
\xb_i & \mbox{if } \# \leq \ms  \quad\mbox{ (source examples)} \\
\xbf'_{\#-\ms} &\mbox{otherwise.} \quad\ \mbox{ (target examples)}
\end{cases} 
$$

The kernel trick allows us to work with
dual weight vector $\alphab\in \R^{\m}$ that is a linear classifier in an augmented space.  Given a kernel $k:\R^d \times \R^d\rightarrow\R$, we have
\begin{equation*}
  h_\wb(\cdot) \ = \ \sgn\left[\sum_{i=1}^\m \alpha_i k(\xb_i, \cdot) \right].
\end{equation*}
Let us denote $K$ the kernel matrix of size $\m\times \m$ such as
$K_{i,j} = k(\xb_i, \xb_j)\,.$
In that case, the objective function
---Equation~(13) of the main paper---can be rewritten in term of the vector 
$$\alphab = (\alpha_1,\alpha_2, \ldots\alpha_{\m})$$
 as
\begin{align*} 
C \times  \sum_{i=\ms}^\m 
\Phi \left(  \frac{\sum_{j=1}^{\m} \alpha_j K_{i,j}}{ \sqrt{K_{i,i}} }  \right)  
\Phi \left(  -\frac{\sum_{j=1}^{\m} \alpha_j K_{i,j}}{ \sqrt{K_{i,i}} }  \right) 
\\
+B \times\sum_{i=1}^\ms 
 \left[
 \Phi \left(  y_i \frac{\sum_{j=1}^{\m} \alpha_j K_{i,j}}{ \sqrt{K_{i,i}} }  \right)
 \right]^2 
+\sum_{i=1}^{\m} \sum_{j=1}^{\m} \alpha_i \alpha_j K_{i,j} \,.
\end{align*}

For our experiments, we minimize this objective function using a \emph{Broyden-Fletcher-Goldfarb-Shanno method (BFGS)} implemented in the \emph{scipy} python library \cite{scipy}.

We initialize the optimization procedure at  $\alpha_i \!=\! \frac1\m$ for all $i\in\{1,\ldots,\m\}$.

\section{Experimental Protocol}

For obtaining the {\algo}$^{RCV}$ results of Table~1, the reverse validation procedure searches on a $20\times 20$ parameter grid for a $C$ between $0.01$ and $10^6$ and a parameter $B$ between $1.0$ and $10^8$, both on a logarithm scale. The results of the other algorithms are reported from~\citet{pbda}.

\end{document}